\documentclass{article}
\usepackage{hyperref}
\usepackage{url}
\usepackage{times}
\usepackage[algo2e]{algorithm2e}

\usepackage{fullpage}
\usepackage{amsmath,amsfonts,amsthm,amssymb}
\usepackage{bbm}
\usepackage{graphics, graphicx, xcolor}
\usepackage{enumitem}
\usepackage{verbatim}		
\usepackage{stmaryrd}
\usepackage[mathscr]{euscript}

\newtheorem{thm}{Theorem}
\newtheorem{lem}[thm]{Lemma}
\newtheorem{prop}[thm]{Proposition}

\DeclareMathOperator{\Prtxt}{Pr}

\newcommand{\allbibdir}{../all-refs/master-bib}

\newcommand{\RR}{\mathbb{R}}      

\newcommand{\evp}[2]{\mathbb{E}_{#2} \left[#1\right]} 
\newcommand{\abs}[1]{\left| #1 \right|}

\newcommand{\prp}[2]{\Prtxt_{#2} \left(#1\right)}

\newcommand{\expp}[1]{\exp \left(#1\right)}

\newcommand{\cH}{\mathcal{H}}

\newcommand{\cF}{\mathcal{F}}

\newcommand{\cO}[1]{\mathcal{O}\left(#1\right)}

\newcommand{\lrp}[1]{\left(#1\right)}
\newcommand{\lrb}[1]{\left[#1\right]}

\newcommand{\condcomment}[2]{\ifthenelse{#1}{#2}{}}

\allowdisplaybreaks

\makeatletter
\def\blfootnote{\xdef\@thefnmark{}\@footnotetext}
\makeatother

\title{PAC-Bayes Iterated Logarithm Bounds for Martingale Mixtures}

 \author{
  Akshay Balsubramani \\
  University of California, San Diego\\
  \url{abalsubr@cs.ucsd.edu}
}

\begin{document}

\maketitle

\begin{abstract}
We give tight concentration bounds for mixtures of martingales that are simultaneously uniform over 
(a) mixture distributions, in a PAC-Bayes sense; and (b) all finite times. 
These bounds are proved in terms of the martingale variance, extending classical Bernstein inequalities, 
and sharpening and simplifying prior work. 
\end{abstract}

\section{Introduction}

The concentration behavior of a martingale $M_t$ -- 
a discrete-time stochastic process with conditionally stationary increments -- 
is well-known to have many applications in modeling sequential processes and algorithms, 
and so it is of interest to analyze for applications in machine learning and statistics. 
It is a long-studied phenomenon \cite{D10}; despite their mighty generality, 
martingales exhibit essentially the same well-understood concentration behavior as simple random walks. 

Even more powerful concentration results can be obtained by considering aggregates of many martingales. 
Though these too have long been studied asymptotically \cite{T75}, 
their non-asymptotic study was only initiated by recent paper of Seldin et al. \cite{SLCBSTA12}, 
which proves the sharpest known results on concentration of martingale mixtures, 
\emph{uniformly} over the mixing distribution in a ``PAC-Bayes" sense which is essentially optimal for such bounds \cite{Ban06}. 
This is motivated by and originally intended for applications in learning theory, as further discussed in that paper. 

In this manuscript, we simplify, strengthen, and subsume the results of \cite{SLCBSTA12}. 
While that paper follows classical central-limit-theorem-type concentration results in focusing on an arbitrary fixed time, 
we instead leverage a recent method in Balsubramani \cite{B15} to achieve concentration that is uniform over finite times, 
extending the law of the iterated logarithm (LIL), with a rate at least as good as \cite{SLCBSTA12} and often far superior. 

In short, our bounds on mixtures of martingales are uniform both over the mixing distribution in a PAC-Bayes sense, 
and over finite times, all simultaneously (and optimally). 
This has no precedent in the literature.


\subsection{Preliminaries}

To formalize this, consider a a set of discrete-time stochastic processes $\{M_t (h)\}_{h \in \cH}$, where $\cH$ is possibly uncountable, 
in a filtered probability space $(\Omega, \cF, \{\cF_t\}_{t \geq 0}, P)$. 
\footnote{As in \cite{B15}, 
we just consider discrete time for convenience; 
the core results and proofs in this paper extend to continuous time as well, 
as well as other generalizations discussed in that paper.}
Define the corresponding difference sequences $\xi_t (h) = M_t (h) - M_{t-1} (h)$ 
(which are $\cF_t$-measurable for any $h, t$)
and conditional variance processes $V_t (h) = \sum_{i=1}^t \evp{\xi_i^2 (h) \mid \cF_{i-1}}{}$.
The mean of the processes at time $t$ w.r.t. any distribution $\rho$ over $\cH$ 
(any $\rho \in \Delta (\cH)$, as we write) is denoted by 
$\evp{M_t}{\rho} := \evp{M_t (h)}{h \sim \rho}$, with $\evp{V_t}{\rho}$ being defined similarly. 
For brevity, we drop the index $h$ when it is clear from context.

Also recall the following standard definitions from martingale theory. 
For any $h \in \cH$, a martingale $M_t$ (resp. supermartingale, submartingale) has $\evp{\xi_t \mid \cF_{t-1}}{} = 0$ (resp. $\leq 0, \;\geq 0$) for all $t$. 
A stopping time $\tau$ is a function on $\Omega$ such that $\{ \tau \leq t \} \in \cF_t$ for all $ t$; 
notably, $\tau$ can be infinite with positive probability (\cite{D10}). 

It is now possible to state our main result.

\begin{thm}[PAC-Bayes Martingale Bernstein LIL Concentration]
\label{thm:pbunifub}
Let $\{M_t (h)\}_{h \in \cH}$ be a set of martingales and fix a distribution $\pi \in \Delta (\cH)$ and a $ \delta < 1$. 
Suppose the difference sequence is uniformly bounded: 
$\abs{M_t (h) - M_{t-1} (h)} \leq e^2$ for all $t \geq 1$ and $h \in \cH$ w.p. $1$. 
Then with probability $\geq 1 - \delta$, the following is true for all $\rho \in \Delta (\cH)$. 

Suppose $\tau_0 (\rho) = \min \left\{s : 2 (e-2) \evp{V_s}{\rho} \geq \frac{2}{\lambda_0^2} \lrp{ \ln \lrp{\frac{4}{\delta}} + KL(\rho \mid\mid \pi) } \right\}$. 
For all $t \geq \tau_0 (\rho)$ simultaneously, 
$\abs{\evp{M_t}{\rho}} \leq \frac{2 (e-2)}{e^2 \lrp{1 + \sqrt{1/3}}} \evp{V_t}{\rho}$ and 
$$ \abs{\evp{M_t}{\rho}} \leq \sqrt{6 (e-2) \evp{V_t}{\rho} \lrp{ 2 \ln \ln \lrp{\frac{3 (e-2) \evp{V_t}{\rho}}{ \abs{\evp{M_t}{\rho} } }} 
+ \ln \left( \frac{2}{\delta} \right) + KL(\rho \mid\mid \pi) }} $$
\end{thm}

(This bound is implicit in $\abs{\evp{M_t}{\rho}}$, but notice that either $\abs{\evp{M_t}{\rho}} \leq 1$ or the iterated-logarithm term can be 
simply treated as $2 \ln \ln \lrp{3 (e-2) \evp{V_t}{\rho}}$, making the bound explicit. )

As we mentioned, Theorem \ref{thm:pbunifub} is uniform over $\rho$ and $t$, 
allowing us to track mixtures of martingales tightly as they evolve. 
The martingales indexed by $\cH$ can be highly dependent on each other, 
as they share a probability space. 
For instance, $M_{t_0} (h_0)$ can depend \emph{in arbitrary fashion} on $\{M_{t \leq t_0} (h)\}_{h \neq h_0}$; 
it is only required to satisfy the martingale property for $h_0$, 
as further discussed in \cite{SLCBSTA12}. 
So these inequalities have found use in analyzing sequentially dependent processes 
such as those in reinforcement learning and bandit algorithms, 
where the choice of prior $\pi$ can be tailored to the problem \cite{SASTOL11} 
and the posterior can be updated based on information learned up to that time.

The method of proof is essentially that used in \cite{B15}. 
Our main observation in this manuscript is that this proof technique is, in a technical sense, 
quite complementary to the fundamental method used in all PAC-Bayes analysis \cite{Ban06}. 
This allows us to prove our results in a sharper and more streamlined way than previous work \cite{SLCBSTA12}.

\subsection{Discussion}
Let us elaborate on these claims. 
Theorem \ref{thm:pbunifub} can be compared directly to the following PAC-Bayes Bernstein bound from \cite{SLCBSTA12} 
that holds for a fixed time: 
\begin{thm}[Restatement\footnote{The actual statement of the theorem in \cite{SLCBSTA12}, 
though not significantly different, is more complicated because of a few inconvenient artifacts of that paper's more complicated analysis, 
none of which arise in our analysis.} of Thm. 8 from \cite{SLCBSTA12}]
\label{thm:oldpbbmart}
Fix any $t$. 
For $\rho$ such that $KL(\rho \mid\mid \pi)$ is sufficiently low compared to $\evp{V_t}{\rho}$, 
\begin{align*}
\abs{\evp{M_t}{\rho}} \leq \sqrt{(1+e)^2 (e-2) \evp{V_t}{\rho} \lrp{ \ln \ln \lrp{\frac{(e-2) t}{ \ln (2/\delta) }} 
+ \ln \left( \frac{2}{\delta} \right) + KL(\rho \mid\mid \pi) }}
\end{align*}
\end{thm} 

This bound is inferior to Theorem \ref{thm:pbunifub} in two significant ways: 
it holds for a fixed time rather than uniformly over finite times, 
and has an iterated-logarithm term of $\ln \ln t$ rather than $\ln \ln V_t$. 
The latter is a very significant difference when $V_t \ll t$, 
which is precisely when Bernstein inequalities would be preferred to more basic inequalities like Chernoff bounds. 

Put differently, our non-asymptotic result, like those of Balsubramani \cite{B15}, 
adapts correctly to the scale of the problem. 
We say ``correctly" because Theorem \ref{thm:pbunifub} is optimal by the (asymptotic) martingale LIL, 
e.g. the seminal result of Stout \cite{S70}; 
this is true non-asymptotically too, by the main anti-concentration bound of \cite{B15}. 
All these optimality results are for a single martingale, but suffice for the PAC-Bayes case as well; 
and the additive $KL(\rho \mid\mid \pi)$ cost of uniformity over $\rho$ is unimprovable in general also, 
by standard PAC-Bayes arguments.


\subsubsection{Proof Overview}

Our method follows that of Balsubramani \cite{B15}, departing from the more traditionally learning-theoretic 
techniques used in \cite{SLCBSTA12}. 
We embark on the proof by introducing a standard exponential supermartingale construction 
that holds for any of the martingales $\{M_t (h)\}_{h \in \cH}$.

\begin{lem}
\label{lem:bexpsupermart}
Suppose $\abs{\xi_t} \leq e^2$ a.s. for all $t$. 
Then for any $h \in \cH$, 
the process $\displaystyle X_t^\lambda (h) := \expp{ \lambda M_t (h) - \lambda^2 (e-2) V_t (h) }$ is a supermartingale in $t$
for any $\lambda \in \left[ - \frac{1}{e^2} , \frac{1}{e^2} \right]$.
\end{lem}
\begin{proof}
It can be checked that $e^x \leq 1 + x + (e-2) x^2$ for $x \leq 1$. 
Then for any $\lambda \in \left[ - \frac{1}{e^2} , \frac{1}{e^2} \right]$ and $t \geq 1$, 
\begin{align*}
\evp{\expp{\lambda \xi_t} \mid \cF_{t-1}}{} &\leq 1 + \lambda \evp{\xi_t \mid \cF_{t-1}}{} + \lambda^2 (e-2) \evp{\xi_t^2 \mid \cF_{t-1}}{} \\
&= 1 + \lambda^2 (e-2) \evp{\xi_t^2 \mid \cF_{t-1}}{} \leq \expp{ \lambda^2 (e-2) \evp{\xi_t^2 \mid \cF_{t-1}}{}}
\end{align*}
using the martingale property on $\evp{\xi_t \mid \cF_{t-1}}{}$. 
Therefore, $\evp{\expp{\lambda \xi_t - \lambda^2 (e-2) \evp{\xi_t^2 \mid \cF_{t-1}}{} } \mid \cF_{t-1}}{} \leq 1$, so $\evp{X_t^\lambda \mid \cF_{t-1}}{} \leq X_{t-1}^\lambda$.
\end{proof}

The classical martingale Bernstein inequality for a given $h$ and fixed time $t$ can be proved 
by using Markov's inequality with $\evp{X_t^{\lambda^*}}{}$, 
where $\lambda^* \propto \frac{\abs{M_t}}{V_t}$ is tuned for the tightest bound, 
and can be thought of as setting the relative scale of variation being measured.

The proof technique of this paper and its advantages over previous work 
are best explained by examining how to set the scale parameter $\lambda$.

\subsubsection{Choosing the Scale Parameter}

On a high level, the main idea of Balsubramani \cite{B15} is to average over a random choice of 
the scale parameter $\lambda$ in the supermartingale $X_t^\lambda$. 
This allows manipulation of a stopped version of $X_t^\lambda$, i.e. $X_{\tau}^\lambda$ for a particular stopping time $\tau$. 
So the averaging technique in \cite{B15} can be thought of as using ``many values of $\lambda$ at once," 
which is necessary when dealing with the stopped process because $\tau$ is random, and so is $\frac{\abs{M_{\tau}}}{V_{\tau}}$. 

All existing PAC-Bayes analyses achieve uniformity in $\rho$ 
through the Donsker-Varadhan variational characterization of the KL divergence:
\begin{lem}[Donsker-Varadhan Lemma (\cite{DV76})]
\label{lem:donskervaradhan}
Suppose $\rho$ and $\pi$ are probability distributions over $\cH$, 
and let $f(\cdot) \colon \cH \mapsto \RR$ be a measurable function. 
Then $$\evp{f(h)}{\rho} \leq KL(\rho \mid\mid \pi) + \ln \lrp{\evp{e^{f(h)}}{\pi}}$$
\end{lem}

This introduces a $KL(\rho \mid\mid \pi)$ term into the bounding of $X_t^\lambda$. 
However, the optimum $\lambda^*$ is then dependent on the unknown $\rho$. 
The solution adopted by existing PAC-Bayes martingale bounds (\cite{SLCBSTA12} and variants) is 
again to use ``many values of $\lambda$ at once." 
In prior work, this is done explicitly by taking a union bound over a grid of carefully chosen $\lambda$s. 

Our main technical contribution is to recognize the similarity between these two problems, 
and to use the (tight) stochastic choice of $\lambda$ in \cite{B15} as a solution to both problems at once, 
achieving the optimal bound of Theorem \ref{thm:pbunifub}.

\section{Proof of Theorem \ref{thm:pbunifub}}
\label{sec:pbmartpf}

We now give the complete proof of Theorem \ref{thm:pbunifub}, following the presentation of \cite{B15} closely.

For the rest of this section, define $U_t := 2 (e-2) V_t$, $k := \frac{1}{3}$, and $\lambda_0 := \frac{1}{e^2 \lrp{1 + \sqrt{k}}}$. 
As in \cite{B15}, our proof invokes the Optional Stopping Theorem from martingale theory, 
in particular a version for nonnegative supermartingales that neatly exploits their favorable convergence properties:
\begin{thm}[Optional Stopping for Nonnegative Supermartingales (\cite{D10}, Theorem 5.7.6)]
\label{thm:optstoppingsupermart}
If $M_t$ is a nonnegative supermartingale and $\tau$ a (possibly infinite) stopping time, $\evp{M_\tau}{} \leq \evp{M_0}{}$.
\end{thm}

%

We also use the exponential supermartingale construction of Lemma \ref{lem:bexpsupermart}, 
which we assume to hold for $M_t (h) \;\forall h \in \cH$ since they are all martingales whose concentration we require. 

Our goal is to control $\evp{M_t}{\rho}$ in terms of $\evp{U_t}{\rho}$, 
so it is tempting to try to show that $e^{\lambda \evp{M_t}{\rho} - \frac{\lambda^2}{2} \evp{U_t}{\rho} }$ 
is an exponential supermartingale. 
However, this is not generally true; 
and even if it were, would only control $\evp{ e^{\lambda \evp{M_{\tau}}{\rho} - \frac{\lambda^2}{2} \evp{U_{\tau}}{\rho} } }{}$ 
for a fixed $\rho$, not in a PAC-Bayes sense. 

We instead achieve uniform control over $\rho$ by using the Donsker-Varadhan lemma to mediate 
the Optional Stopping Theorem every time it is used in Balsubramani's proof \cite{B15} of the single-martingale case. 
This is fully captured in the following key result, 
a powerful extension of a standard moment-generating function bound that is uniform in $\rho$ 
and has enough freedom (an arbitrary stopping time $\tau$) to be converted into a time-uniform bound.

\begin{lem}
\label{lem:dvunif}
Choose any probability distribution $\pi$ over $\cH$. 
Then for any stopping time $\tau$ and $\lambda \in \left[ - \frac{1}{e^2} , \frac{1}{e^2} \right]$, 
simultaneously for all distributions $\rho$ over $\cH$, 
$$ \evp{ e^{\lambda \evp{M_{\tau}}{\rho} - \frac{\lambda^2}{2} \evp{U_{\tau}}{\rho} } }{} \leq e^{KL(\rho \mid\mid \pi)} $$
\end{lem}
\begin{proof}
Using Lemma \ref{lem:donskervaradhan} with the function $f(h) = \lambda M_{\tau} (h) - \frac{\lambda^2}{2} U_{\tau} (h)$, 
and exponentiating both sides, 
we have for all posterior distributions $\rho \in \Delta (\cH)$ that
\begin{align}
\label{eq:dvforvarmart}
e^{ \lambda \evp{M_{\tau}}{\rho} - \frac{\lambda^2}{2} \evp{U_{\tau}}{\rho} }
\leq e^{KL(\rho \mid\mid \pi)} \evp{e^{\lambda M_{\tau} - \frac{\lambda^2}{2} U_{\tau} }}{\pi}
\end{align}
Therefore,
$ \displaystyle
\evp{ e^{\lambda \evp{M_{\tau}}{\rho} - \frac{\lambda^2}{2} \evp{U_{\tau}}{\rho} } }{}
\stackrel{(a)}{\leq} e^{KL(\rho \mid\mid \pi)} \evp{ \evp{ e^{\lambda M_{\tau} - \frac{\lambda^2}{2} U_{\tau} } }{}}{\pi} 
\stackrel{(b)}{\leq} e^{KL(\rho \mid\mid \pi)}
$
where $(a)$ is from \eqref{eq:dvforvarmart} and Tonelli's Theorem, 
and $(b)$ is by Lemma \ref{lem:bexpsupermart} and Optional Stopping (Thm. \ref{thm:optstoppingsupermart}). 
\end{proof}

Just as a bound on the moment-generating function of a random variable is the key 
to proving tight Hoeffding and Bernstein concentration of that variable,
Lemma \ref{lem:dvunif} is, exactly analogously, the key tool used to prove Theorem \ref{thm:pbunifub}.

\subsection{A PAC-Bayes Uniform Law of Large Numbers}
First, we define the stopping time 
$\tau_0 (\rho) := \min \left\{t : \evp{U_t}{\rho} \geq \frac{2}{\lambda_0^2} \lrp{ \ln \lrp{\frac{2}{\delta}} + KL(\rho \mid\mid \pi) } \right\}$ 
and the following ``good" event:
\begin{align}
\label{defofB}
B_{\delta} = 
\left\{ \omega \in \Omega : \forall \rho \in \Delta(\cH), \quad \frac{\abs{\evp{M_t}{\rho}}}{\evp{U_t}{\rho}} \leq \lambda_0 
\qquad \forall t \geq \tau_0 (\rho) \right\} 
\end{align}

Our first result introduces the reader to our main proof technique; 
it is a generalization of the law of large numbers (LLN) to our PAC-Bayes martingale setting. 
\begin{thm}
\label{thm:varunifpbub}
Fix any $\delta > 0$. With probability $\geq 1-\delta$, 
the following is true for all $\rho$ over $\cH$: 
for all $t \geq \tau_0 (\rho)$, 
$$ \frac{\abs{\evp{M_t}{\rho}}}{\evp{U_t}{\rho}} \leq \lambda_0 $$
\end{thm}

To prove this, we first manipulate Lemma \ref{lem:dvunif} so that it is in terms of $\abs{\evp{M_\tau}{\rho}}$.
\begin{lem}
\label{lem:btstrpmgfpb}
Choose any prior $\pi \in \Delta (\cH)$. For any stopping time $\tau$ and all distributions $\rho$ over $\cH$, 
$$ \evp { \expp{ \lambda_0 \abs{\evp{M_\tau}{\rho}} - \frac{\lambda_0^2}{2} \evp{U_\tau}{\rho}} }{} \leq 2 e^{KL(\rho \mid\mid \pi)}$$
\end{lem}
\begin{proof}
Lemma \ref{lem:dvunif} describes the behavior of the process $\chi_t^\lambda = e^{ \lambda \evp{M_{t}}{\rho} - \frac{\lambda^2}{2} \evp{U_{t}}{\rho} }$. 
Define $Y_t$ to be the mean of $\chi_t^\lambda$ 
with $\lambda$ being set \emph{stochastically}: $\lambda \in \{ - \lambda_0, \lambda_0 \}$ with probability $\frac{1}{2}$ each. 
After marginalizing over $\lambda$, the resulting process is
\begin{align}
\label{btstrpmgf}
Y_t &= \frac{1}{2} \expp{ \lambda_0 \evp{M_t}{\rho} - \frac{\lambda_0^2}{2} \evp{U_t}{\rho}} + \frac{1}{2} \expp{ - \lambda_0 \evp{M_t}{\rho} - \frac{\lambda_0^2}{2} \evp{U_t}{\rho}} \geq \frac{1}{2} \expp{ \lambda_0 \abs{\evp{M_t}{\rho}} - \frac{\lambda_0^2}{2} \evp{U_t}{\rho}}
\end{align}
Now take $\tau$ to be any stopping time as in the lemma statement. 
Then $\evp{\expp{ \lambda_0 \evp{M_\tau}{\rho} - \frac{\lambda_0^2}{2} \evp{U_\tau}{\rho}} }{} \leq e^{KL(\rho \mid\mid \pi)}$, 
by Lemma \ref{lem:dvunif}. 
Similarly, $\evp{X_\tau^{\lambda = -\lambda_0} }{} \leq e^{KL(\rho \mid\mid \pi)}$.

So $\evp{Y_\tau}{} = \frac{1}{2} \lrp{ \evp{ X_\tau^{\lambda = -\lambda_0}}{} + \evp{X_\tau^{\lambda = \lambda_0}}{} } \leq e^{KL(\rho \mid\mid \pi)}$. 
Combining this with \eqref{btstrpmgf} gives the result.
\end{proof}

A particular setting of $\tau$ extracts the desired uniform LLN bound from Lemma \ref{lem:btstrpmgfpb}.

\begin{proof}[Proof of Theorem \ref{thm:varunifpbub}]
Define the stopping time 
$$ \tau = \min \left\{ t : \exists \rho \in \Delta (\cH) \text{ s.t. } 
t \geq \tau_0 (\rho) \text{ and } \frac{\abs{\evp{M_t}{\rho}}}{\evp{U_t}{\rho}} > \lambda_0 \right\} $$
Then it suffices to prove that $P(\tau < \infty) \leq \delta$. 

On the event $\{ \tau < \infty \}$, we have for some $\rho$ that 
$\frac{\abs{\evp{M_\tau}{\rho}}}{\evp{U_\tau}{\rho}} > \lambda_0$ by definition of $\tau$. 
Therefore, using Lemma \ref{lem:btstrpmgfpb}, we can conclude that for this $\rho$, 
\begin{align*}
2 e^{KL(\rho \mid\mid \pi)} &\geq \evp { \expp{ \lambda_0 \abs{\evp{M_\tau}{\rho}} - \frac{\lambda_0^2}{2} \evp{U_\tau}{\rho}} \mid \tau < \infty}{} P(\tau < \infty)
\stackrel{(b)}{\geq} \frac{2}{\delta} e^{KL(\rho \mid\mid \pi)} P(\tau < \infty)
\end{align*}
where $(b)$ uses $\evp{U_\tau}{\rho} \geq \evp{U_{\tau_0}}{\rho} \geq \frac{2}{\lambda_0^2} \ln \lrp{\frac{2}{\delta} e^{KL(\rho \mid\mid \pi)}}$. 
Therefore, $P(\tau < \infty) \leq \delta$, as desired. 
\end{proof}

\subsection{Proof of Theorem \ref{thm:pbunifub}}

For any event $E \subseteq \Omega$ of nonzero measure, let $\evp{\cdot}{E}$ denote the expectation restricted to $E$, 
i.e. $\evp{f}{E} = \frac{1}{P(E)} \int_E f(\omega) P(d \omega)$ for a measurable function $f$ on $\Omega$. 
Similarly, dub the associated measure $P_E$, where for any event $\Xi \subseteq \Omega$ we have $P_E(\Xi) = \frac{P(E \cap \Xi)}{P(E)}$.

Theorem \ref{thm:varunifpbub} shows that $P (B_\delta) \geq 1-\delta$.
It is consequently easy to observe that the corresponding restricted outcome space can still be used to control expectations.
\begin{prop}
\label{prop:shiftoutofB}
For any nonnegative r.v. $X$, we have 
$\evp{X}{B_\delta} \leq \frac{1}{1-\delta} \evp{X}{}$.
\end{prop}
\begin{proof}
Using Thm. \ref{thm:varunifpbub}, $\evp{X}{} = \evp{X}{B_\delta} P(B_\delta) + \evp{X}{B_\delta^c} P(B_\delta^c) \geq \evp{X}{B_\delta} (1-\delta)$.
\end{proof}

Just as in \cite{B15}, 
the idea of the main proof is to choose $\lambda$ stochastically from a probability space $(\Omega_\lambda, \cF_\lambda, P_\lambda)$ such that 
$\displaystyle P_\lambda (d \lambda) = \frac{d \lambda}{\abs{\lambda} \lrp{ \log \frac{1}{\abs{\lambda}} }^2} $ on $\lambda \in [-e^{-2}, e^{-2}] \setminus \{ 0 \}$. 
The parameter $\lambda$ is chosen independently of $\xi_1, \xi_2, \dots$, 
so that $X_t^\lambda$ is defined on the product space.
Write $\mathbb{E}^{\lambda} \left[ \cdot \right]$ to denote the expectation with respect to $(\Omega_\lambda, \cF_\lambda, P_\lambda)$. 

To be consistent with previous notation, 
we continue to write $\evp{\cdot}{}$ to denote the expectation w.r.t. the original probability space $(\Omega, \cF, P)$. 
As mentioned earlier, we use subscripts for expectations conditioned on events in this space, e.g. $\evp{X}{A_\delta}$. 
(As an example, $\evp{\cdot}{\Omega} = \evp{\cdot}{}$.)

Before going on to prove the main theorem, we need one more result that controls the effect of averaging over $\lambda$ as above. 
\begin{lem}
\label{lem:intlbdapb}
For all $\rho \in \Delta (\cH)$ and any $\delta$, the following is true: 
for any stopping time $\tau \geq \tau_0 (\rho)$,
$$ \evp{\mathbb{E}^{\lambda} \left[ e^{\lambda \evp{M_{\tau}}{\rho} - \frac{\lambda^2}{2} \evp{U_{\tau}}{\rho} } \right]}{B_{\delta}} 
\geq \evp{ \frac{ 2 \expp{ \frac{\evp{M_{\tau}}{\rho}^2}{2 \evp{U_{\tau}}{\rho} } (1-k) } } 
{ \ln^2 \lrp{\frac{\evp{U_{\tau}}{\rho} }{ \lrp{1 - \sqrt{k}} \abs{ \evp{M_{\tau}}{\rho} } } } } }{B_\delta} $$
\end{lem}

Lemma \ref{lem:intlbdapb} is precisely analogous to Lemma 13 in \cite{B15} and 
proved using exactly the same calculations, so its proof is omitted here. 
Now we can prove Theorem \ref{thm:pbunifub}.

\begin{proof}[Proof of Theorem \ref{thm:pbunifub}]
The proof follows precisely the same method as that of Balsubramani \cite{B15}, 
but with a more nuanced setting of the stopping time $\tau$. 
To define it, we first for convenience define the deterministic function 
\begin{align*}
\displaystyle \zeta_t (\rho) = 
\sqrt{ \frac{2 \evp{U_t}{\rho}}{1-k} \ln \left( \frac{ 2 \ln^2 \lrp{\frac{\evp{U_t}{\rho}}{ \lrp{1 - \sqrt{k}} \abs{ \evp{M_t}{\rho}} } } e^{KL(\rho \mid\mid \pi)} }
{\delta} \right)}
\end{align*}
Now we define the stopping time 
\begin{align*}
\displaystyle \tau = \min \Bigg\{  t : \;\;&\exists \rho \in \Delta (\cH) \;\text{ s.t. }\; 
t \geq \tau_0 (\rho) \quad \text{ and } \\
&\lrb{ \abs{\evp{M_t}{\rho}} > \lambda_0 \evp{U_t}{\rho}  \;\;\text{ or }\;\; 
\left( \abs{\evp{M_t}{\rho}} \leq \lambda_0 \evp{U_t}{\rho} \;\text{ and }\; 
\abs{\evp{M_t}{\rho}} >  \zeta_t (\rho)  \right) } \Bigg\}
\end{align*}
The rest of the proof shows that $P (\tau = \infty) \geq 1 - \delta$, 
and involves nearly identical calculations to the main proof of \cite{B15} 
(with $\evp{M_t}{\rho}, \evp{U_t}{\rho}, B_{\delta}$ replacing what that paper writes as $M_t, U_t, A_{\delta}$). 

It suffices to prove that $P (\tau = \infty) \geq 1 - \delta$.
On the event $\left\{ \{ \tau < \infty \} \cap B_{\delta/2} \right\}$,  
by definition there exists a $\rho$ s.t. 
\begin{align}
\label{eq:condtaupb}
\abs{\evp{M_\tau}{\rho}} &> \zeta_t (\rho) 
= \sqrt{ \frac{2 \evp{U_\tau}{\rho} }{1-k} \ln \left( \frac{ 2 \ln^2 \lrp{\frac{\evp{U_\tau}{\rho} }{ \lrp{1 - \sqrt{k}} \abs{ \evp{M_\tau}{\rho}} } } }{\delta} 
e^{KL(\rho \mid\mid \pi)} \right)} 
\nonumber \\
&\text{I.e. a $\rho$ such that } \qquad \frac{ 2 \expp{ \frac{\evp{M_\tau}{\rho}^2}{2 \evp{U_\tau}{\rho} } (1-k) } } 
{ \ln^2 \lrp{\frac{\evp{U_\tau}{\rho} }{ \lrp{1 - \sqrt{k}} \abs{ \evp{M_\tau}{\rho}} } } } 
> \frac{4}{\delta} e^{KL(\rho \mid\mid \pi)}
\end{align} 

Consider this $\rho$.  
Using the nonnegativity of 
$\frac{ 2 \expp{ \frac{\evp{M_t}{\rho}^2}{2 \evp{U_t}{\rho} } (1-k) } } { \ln^2 \lrp{\frac{\evp{U_t}{\rho} }{ \lrp{1 - \sqrt{k}} \abs{ \evp{M_t}{\rho} } } } }$ on $B_{\delta/2}$ 
and letting $Z_\tau^\lambda := e^{\lambda \evp{M_{\tau}}{\rho} - \frac{\lambda^2}{2} \evp{U_{\tau}}{\rho} }$, 
\begin{align*}
2 e^{KL(\rho \mid\mid \pi)} &\geq \frac{e^{KL(\rho \mid\mid \pi)} }{1-\frac{\delta}{2}} 
\stackrel{(a)}{\geq} \frac{ \mathbb{E}^{\lambda} \left[ \evp{ Z_\tau^\lambda }{} \right] }{1-\frac{\delta}{2}} 
\stackrel{(b)}{\geq} \mathbb{E}^{\lambda} \left[ \evp{ Z_\tau^\lambda }{B_{\delta/2}} \right]
\stackrel{(c)}{=} \evp{\mathbb{E}^{\lambda} \left[ Z_\tau^\lambda \right]}{B_{\delta/2}} \\
&\stackrel{(d)}{\geq} \evp{ \frac{ 2 \expp{ \frac{\evp{M_\tau}{\rho}^2}{2 \evp{U_\tau}{\rho} } (1-k) } } 
{ \ln^2 \lrp{\frac{\evp{U_\tau}{\rho} }{ \lrp{1 - \sqrt{k}} \abs{ \evp{M_\tau}{\rho}}} } } }{B_{\delta/2}} 
\geq \evp{ \frac{ 2 \expp{ \frac{\evp{M_\tau}{\rho}^2}{2 \evp{U_\tau}{\rho} } (1-k) } } 
{ \ln^2 \lrp{\frac{\evp{U_\tau}{\rho} }{ \lrp{1 - \sqrt{k}} \abs{ \evp{M_\tau}{\rho}}} } } \mid \tau < \infty}{B_{\delta/2}} 
P_{B_{\delta/2}}(\tau < \infty) \\
&\stackrel{(e)}{>} \frac{4}{\delta} e^{KL(\rho \mid\mid \pi)} P_{B_{\delta/2}}(\tau < \infty)
\end{align*}
where $(a)$ is by Lemma \ref{lem:dvunif}, 
$(b)$ is by Prop. \ref{prop:shiftoutofB}, 
$(c)$ is by Tonelli's Theorem, 
$(d)$ is by Lemma \ref{lem:intlbdapb}, 
and $(e)$ is by \eqref{eq:condtaupb}.

After simplification, this gives 
\begin{align}
P_{B_{\delta/2}} (\tau < \infty) \leq \delta/2 \implies P_{B_{\delta/2}} (\tau = \infty) \geq 1 - \frac{\delta}{2}
\end{align} 
and using Theorem \ref{thm:varunifpbub} -- that $P(B_{\delta/2}) \geq 1 - \frac{\delta}{2}$ -- concludes the proof.
\end{proof}

\newpage
\bibliographystyle{plain}
\bibliography{\allbibdir}

\begin{thebibliography}{1}

\bibitem{B15}
Akshay Balsubramani.
\newblock Sharp uniform martingale concentration bounds.
\newblock {\em arXiv preprint arXiv:1405.2639}, 2015.

\bibitem{Ban06}
Arindam Banerjee.
\newblock On bayesian bounds.
\newblock In {\em Proceedings of the 23rd international conference on Machine
  learning}, pages 81--88. ACM, 2006.

\bibitem{DV76}
M.~D. Donsker and S.~S. Varadhan.
\newblock Asymptotic evaluation of certain markov process expectations for
  large time. iii.
\newblock {\em Communications on Pure and Applied Mathematics}, 28:389--461,
  1976.

\bibitem{D10}
Rick Durrett.
\newblock {\em Probability: Theory and Examples}.
\newblock Cambridge Series in Statistical and Probabilistic Mathematics.
  Cambridge University Press, 4th edition, 2010.

\bibitem{SASTOL11}
Yevgeny Seldin, Peter Auer, John~S Shawe-taylor, Ronald Ortner, and
  Fran{\c{c}}ois Laviolette.
\newblock Pac-bayesian analysis of contextual bandits.
\newblock In {\em Advances in Neural Information Processing Systems}, pages
  1683--1691, 2011.

\bibitem{SLCBSTA12}
Yevgeny Seldin, Fran{\c{c}}ois Laviolette, Nicolo Cesa-Bianchi, John
  Shawe-Taylor, and Peter Auer.
\newblock Pac-bayesian inequalities for martingales.
\newblock {\em Information Theory, IEEE Transactions on}, 58(12):7086--7093,
  2012.

\bibitem{S70}
William~F. Stout.
\newblock A martingale analogue of kolmogorov's law of the iterated logarithm.
\newblock {\em Z. Wahrsch. Verw. Gebiete}, 15:279--290, 1970.

\bibitem{T75}
R.~J. Tomkins.
\newblock Iterated logarithm results for weighted averages of martingale
  difference sequences.
\newblock {\em Ann. Probab.}, 3(2):307--314, 04 1975.

\end{thebibliography}

\end{document}